\documentclass{article}

\usepackage{microtype}
\usepackage{graphicx}
\usepackage{subcaption}
\usepackage{booktabs} 

\usepackage{hyperref}

\usepackage[accepted]{icml2026}

\usepackage[table]{xcolor}
\usepackage{booktabs}
\usepackage{xcolor}
\usepackage{siunitx}
\usepackage{tabularray}
\usepackage{multirow}
\usepackage{subcaption}

\def\ie{{\em i.e.}}
\def\eg{{\em e.g.}}

\usepackage{amsmath}
\usepackage{amssymb}
\usepackage{mathtools}
\usepackage{amsthm, bm}

\usepackage[capitalize,noabbrev]{cleveref}

\theoremstyle{plain}
\newtheorem{theorem}{Theorem}[section]
\newtheorem{proposition}[theorem]{Proposition}

\theoremstyle{definition}
\newtheorem{definition}[theorem]{Definition}
\newtheorem{assumption}[theorem]{Assumption}
\theoremstyle{remark}
\newtheorem{remark}[theorem]{Remark}

\usepackage[textsize=tiny]{todonotes}

\icmltitlerunning{TD-VAD: Breaking Visual Dependence in Video Anomaly Detection with Text-Driven Learning}

\begin{document}

\twocolumn[
  \icmltitle{TD-VAD: Breaking Visual Dependence in Video Anomaly Detection with Text-Driven Learning}



  \icmlsetsymbol{equal}{*}

  \begin{icmlauthorlist}
    \icmlauthor{Shuangqing Zhang}{1nju}
    \icmlauthor{Lei-Lei Ma}{3hfit}
    \icmlauthor{Zhao Wang}{2comp}
    \icmlauthor{Wen Dong}{1nju}
    \icmlauthor{Xinyi Xu}{1nju} \\
    \icmlauthor{Guo-Sen Xie}{4njust}
    \icmlauthor{Caifeng Shan}{1nju}
    \icmlauthor{Fang Zhao}{1nju}
  \end{icmlauthorlist}

  \icmlaffiliation{1nju}{School of Intelligence Science and Technology, Nanjing University, Suzhou 215163, China.}
  \icmlaffiliation{2comp}{China Mobile Zijin Innovation Institute, Nanjing 211800, China.}
  \icmlaffiliation{3hfit}{School of Artificial Intelligence Engineering, Hefei Institute of Technology, Hefei 230001, China.}
  \icmlaffiliation{4njust}{ School of Computer Science and Engineering, Nanjing University of Science and Technology, Nanjing 210014, China}

  \icmlcorrespondingauthor{Fang Zhao}{fzhao@nju.edu.cn}

  \icmlkeywords{Machine Learning, ICML}

  \vskip 0.3in
]



\printAffiliationsAndNotice{}  

\begin{abstract}
  Visual data is typically a prerequisite for training existing video anomaly detection (VAD) methods. However, obtaining sufficient annotated anomaly data for training is challenging and not scalable due to the rarity of anomaly data and the wide variety of abnormal events. In this work, we advocate that the effectiveness of treating texts as video sequences for the VAD model and propose a novel \textbf{T}ext-\textbf{D}riven  \textbf{V}ideo \textbf{A}nomaly \textbf{D}etection (\textbf{TD-VAD}) approach to break visual dependence. In contrast to the anomaly video data, text descriptions of abnormal events are easy to collect, and their class labels can be directly derived. Specifically, our method utilizes video-like text descriptions with temporal characteristics generated by LLM to train a VAD model, without any reliance on target-domain anomaly data. To capture the long and short-range temporal logic of events, we design the event evolution causal attention module to model contextual dependencies across time. 
During inference, considering the domain gap between the texts and video sequences, we use the frozen CLIP encoder to extract embeddings of video frames to align the text modality while retaining crucial visual information. Comprehensive experiments on two large-scale VAD datasets, XD-Violence and UCF-Crime, demonstrate that our method outperforms prior one-class and unsupervised VAD methods by a large margin. 
\end{abstract}

\begin{figure}[t]
  \centering

   \includegraphics[width=1\linewidth]{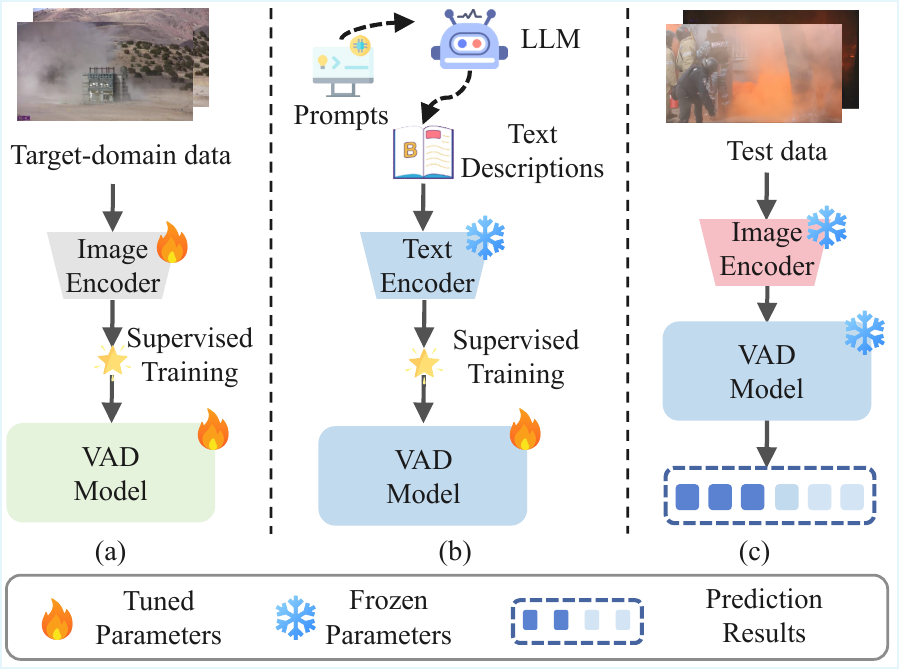}

   \caption{A comparison between prior VAD methods and the proposed approaches. (a) Prior methods use huge annotated domain-specific video data to train the VAD model. (b) The proposed approach only uses the easily-accessible text descriptions of anomaly categories to train the VAD model. (c) During testing, the proposed TD-VAD can be applied to video sequences directly.  }
   
   \label{fig:abs}
  \vspace{-4mm}
\end{figure}

\section{Introduction}
\label{sec:intro}

Video anomaly detection (VAD) aims to identify abnormal events automatically within video sequences and has attracted significant attention due to its overcoming the impracticality of manual monitoring in surveillance security~\cite{sultani2018real}, and industrial inspection~\cite{duan2025anomalycontrol}. 
Prior works have mainly focused on one-class~\cite{hasan2016learning, lu2013abnormal} or weak-supervised~\cite{ tian2021weakly, wu2022self, zhang2026contextual} VAD tasks, wherein the former typically fit a model to the normal videos and treat any deviations from it as anomalous, and the latter needs to provide frame-level anomaly confidences with video-level annotations merely. 

Existing VAD approaches typically assume that training video examples in the target domain are available for learning the VAD models. However, the assumption may not hold in various scenarios, such as i) when training data violates the data privacy policies (\eg, to protect the sensitive information)~\cite{wang2025federated}, or ii) when the target domain does not have relevant video data (\eg, the violence events in the city)~\cite{wu2020not}. Vision-free video anomaly detection is an emerging task for VAD in such scenarios, to which the aforementioned VAD approaches, including weakly-supervised, one-class, or unsupervised paradigms~\cite{thakare2023rareanom, zhang2025autoregressive}, are not viable, as it requires VAD models to detect anomaly events in video sequences without any video example in as target dataset, imposing prohibitive time and cost burdens when target-domain samples are unavailable. The challenge of anomaly data collection led us to explore a novel research question: \textit{Can we develop a VAD method without using visual data in training?}

In this paper, we aim to answer this challenging question. Developing a vision-free VAD model is hard due to the lack of explicit visual priors about anomaly events. Recently, large pre-trained vision-language models (VLMs)~\cite{liu2023visual} and large language models (LLMs)~\cite{touvron2023llama} have demonstrated strong generalization capability in various vision tasks, including anomaly detection~\cite{jeong2023winclip, zhou2023anomalyclip}. Such visual priors for anomaly events might be drawn using large foundation models, renowned for their generalization capability and wide knowledge encapsulation. Generally speaking, compared with the collection of large-scale annotated anomaly video datasets, it is significantly easier to collect large amounts of natural text descriptions generated from LLMs. However, there are several critical challenges that need to be addressed. First, video is a continuous stream of visual frames that relies on temporal dynamics, which hinders the knowledge transfer from the concise texts to the video domain. Second, the modality gap between the text and video domains further degrades the generalization. If we can find a cohesive alignment of semantics across video and text modalities, then we can explore the information in text data to promote the performance of VAD in the video space. The aligned embedding space of pre-trained VLMs provides a possible solution for this challenge. Therefore, we investigate the potential of combining existing LLMs with VLMs in addressing vision-free VAD.

 On top of our preliminary findings, we propose a novel \textbf{T}ext-\textbf{D}riven \textbf{V}ideo \textbf{A}nomaly \textbf{D}etection (\textbf{TD-VAD}) approach to break visual dependence, which treats text descriptions as alternatives to videos to train the VAD model.
The main idea of the proposed framework is illustrated in Fig.~\ref{fig:abs}. Specifically, we exploit an LLM such as DeepSeek-V3~\cite{liu2024deepseek} to generate textual descriptions for all anomaly categories in a target dataset. Compared with collecting textual data on the web, generating data using LLM is more straightforward and requires less manual post-processing. In particular, for the challenge that texts lack temporal information, we use LLM to generate text descriptions with temporal logic that include both precise timestamps and detailed event descriptions to simulate the temporal characteristics of video sequences. Additionally, considering that only training the VAD model with text data will lead to a modality gap between texts and video frames, we take full advantage of the pre-trained CLIP model to extract both the text and visual embeddings to utilize its strong alignment space to reduce the impact of the modality gap. Moreover, to capture the long and short-range temporal logic of events from the generated video-like text descriptions, we design a novel event-evolution causal attention (EC-Attn) module to model contextual dependencies across time. This module contains an event-context causal attention (ECC-Attn) and an event-focus causal attention (EFC-Attn), wherein the former models the long-range temporal logic of full text sequences to capture the evolution chain, and the latter focuses on supplementing short-window action details. Moreover, a hierarchical anomaly-aware branch is designed to determine if an anomaly exists in video frames. 
During testing, we can replace the input from text descriptions with video sequences. The video features extracted from the frozen CLIP image encoder can directly feed into the text-trained VAD model to obtain frame anomaly confidence.


The main contributions can be summarized as follows:

\begin{itemize}
\item {We propose a novel method to train VAD based only from text generated by LLM, which allows the model to be trained with fully supervised data without any video data, to address the pressing issue of data scarcity and privacy in anomaly detection.}

\item {We propose text descriptions as videos to train a model to detect anomaly events in videos. Text descriptions are easily accessible and, in contrast to images, their class labels can be directly derived, making the proposed method compelling in practice.}

\item {We propose a simple yet efficient temporal modeling method based on the pre-trained CLIP model to align the text and video modalities, which can effectively reduce the modality gap.}
\end{itemize}

\section{Related Work}
\label{sec:work}
\textbf{Video Anomaly Detection.} Early video anomaly detection (VAD) studies primarily minimize the reconstruction errors on normal-only video sequences~\cite{hasan2016learning,lu2013abnormal, thakare2023rareanom}, and thus regard large reconstruction errors as a signal of anomaly one. A parallel line formulates VAD as weakly-supervised multiple-instance learning, leveraging video-level labels but no reliable frame labels to localize frame-level anomalies~\cite{sultani2018real, wu2021learning, wu2024vadclip}, avoiding the prohibitive cost of frame annotation. Sultani et al.~\cite{sultani2018real} collect a large-scale VAD dataset with video-level annotation and use a multiple instance ranking strategy to localize temporally anomalous events. Additionally, recent works incorporate audio modality~\cite{wu2020not,wu2025avadclip} or instruction-tunes detectors on the collected video-caption corpus~\cite{zhang2024holmes} to broaden coverage. Recently, there has been a surge in efforts to leverage multi-modal knowledge of vision-language models (VLMs) to boost VAD performance via prompt tuning or lightweight adapters~\cite{ye2025vera, wu2024vadclip, zhang2024holmes, chen2025aligning}. However, these approaches still need target-domain videos for fine-tuning, consuming huge amounts of manual labor to collect and annotate this data. In contrast, the proposed approach requires no additional video data yet matches the accuracy of tuned models with fast inference speed and little GPU costs.

\textbf{Large Foundation Models for VAD.} Large language models (LLMs)~\cite{liu2024deepseek, wang2022self, wei2022chain}, and vision-language models (VLMs)~\cite{alayrac2022flamingo, dai2023instructblip,li2023blip} have been explored and already achieve impressive performance across diverse tasks such as visual question answering. Kim et al.~\cite{kim2023unsupervised} apply the text descriptions generated via LLM for the unsupervised anomaly detection. Recently, LAVAD~\cite{zanella2024harnessing} extends this power to anomaly detection: it generates textual descriptions for each video frame with a VLM and scores each one with an LLM, achieving superior anomaly detection performance. Chen et al.\cite{chen2025aligning} propose a novel method for understanding video anomalies with LLMs by aligning effective tokens in terms of temporal and spatial space. Ye et al.~\cite{ye2025vera} refine prompts to adapt frozen VLMs as an integrated system for VAD.
Later work via adding a temporal graph module~\cite{shao2025eventvad}, or fusing audio cues~\cite{dev2024mcanet} to improve anomaly detection performance. However, the deployment of these aforementioned large models markedly escalates the demand for high-end GPUs compared to conventional ConvNet approaches. 

Different from the above methods, the proposed method only leverages the text descriptions generated by LLM to address temporal anomaly detection on videos to break visual dependence, requiring no video collection and annotation.

\begin{figure*}[t]
  \centering
   \includegraphics[width=1\linewidth]{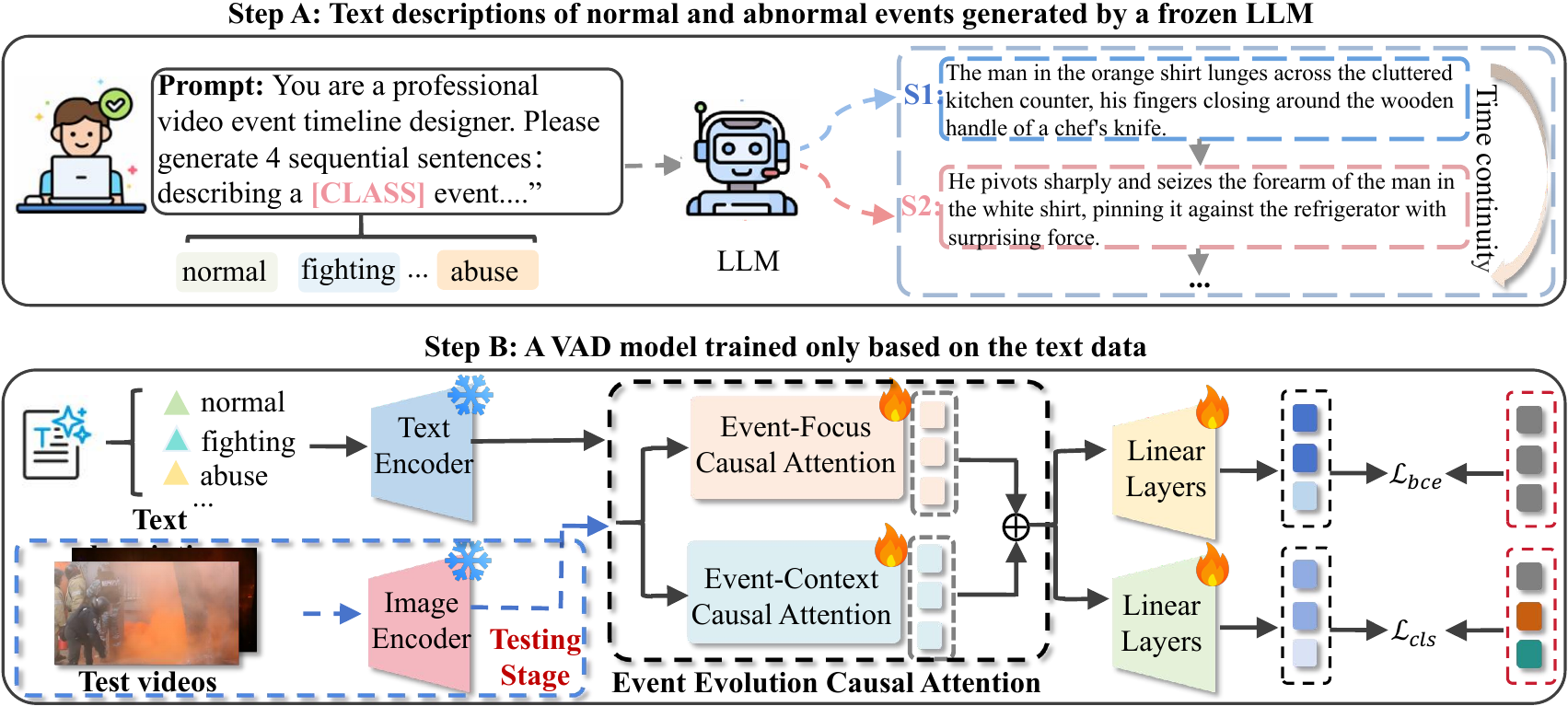}
   \caption{An overview of the proposed TD-VAD approach. The proposed method is based on a pre-trained CLIP containing the frozen text encoder and image encoder. (\textbf{Top}) A frozen LLM~\cite{liu2024deepseek} generates a diverse set of normal and anomalous scene sentences using a prompt. (\textbf{Bottom}) During training, the generated text descriptions are embedded by the frozen CLIP text encoder, and then these embeddings are used to train a VAD model. During testing, we replace the input text descriptions with videos.  }
   \label{fig:framework}
 \vspace{-4mm}
\end{figure*}

\section{Methodology}
\label{sec:method}

\subsection{Problem Formulation}
Let a video be represented as a frame sequence
$V = \{v_t\}_{t=1}^{T_v}$.
The goal of video anomaly detection (VAD) is to learn a scoring function
$\varphi: v_t \mapsto p_t \in [0,1]$,
where $p_t$ denotes the anomaly likelihood of frame $v_t$.
Depending on the level of supervision, prior VAD approaches can be categorized as:
(i) \textit{weakly-supervised} methods, which train on videos with only video-level labels (normal or abnormal);
(ii) \textit{one-class} methods, which train solely on normal videos and detect deviations during inference; and
(iii) \textit{unsupervised} methods, which assume no labels at all.
In contrast, the proposed \textit{vision-free VAD} setting assumes that no
target-domain training videos are available, \ie,
$D_{\mathrm{train}} = \varnothing$.
Instead, the model is provided only with a set of semantic anomaly concepts
$\mathcal{C}=\{c_k\}_{k=1}^K$.
Given an unseen test video $V_{\mathrm{test}}$,
the model infers frame-level anomaly scores
$\{p_t\}_{t=1}^{T_v}$ conditioned solely on textual semantics:
\begin{align}
\varphi_\theta:\ (V_{\mathrm{test}}, \mathcal{C}) \mapsto \{p_t\}_{t=1}^{T_v}.
\end{align}

\subsection{Overview}
The proposed framework is shown in Fig.~\ref{fig:framework}. 
\textbf{In the first step} (\textit{generating text descriptions}), we adopt an LLM, DeepSeek-V3~\cite{liu2024deepseek}, to generate text descriptions with timestamps for anomaly events.  
\textbf{In the second step} (\textit{training a VAD model only using text data}), we use the frozen CLIP text encoder~\cite{radford2021learning} to encode the input text, and present an event evolution causal attention module to capture the evolution chain by modeling contextual dependencies over time. 
This module involves two components, \ie, an event-context causal attention (ECC-Attn) and an event-focus causal attention (EFC-Attn), which capture the global temporal pattern and local action details of anomaly events, respectively. 
More importantly, thanks to the aligned image-text embedding space of CLIP, we can transfer the VAD model trained on text to the visual domain by utilizing the alignment between text and video frames. 
During inference, video frame embeddings are extracted using the frozen CLIP image encoder to replace text input. 
Overall, we propose a simple and efficient method to achieve vision-free VAD using only text data.

\subsection{Generating Text Descriptions via LLM}
A key challenge of vision-free VAD is the absence of annotated video clips for learning the temporal evolution of anomaly events. To enable training without any visual data, we construct a large-scale, temporally structured textual dataset using a Large Language Model (LLM) to reflect real-world visual dynamics. The generated text sequences aim to (i) convey anomaly semantics faithfully, (ii) exhibit clear temporal progression, and (iii) provide balanced coverage across event categories. 

\noindent\textbf{Category Preparation.}
We first collect all normal and abnormal categories from the target dataset, \eg, \textit{fighting}, \textit{shooting}, and \textit{riot} in XD-Violence~\cite{wu2020not}. To prevent distributional bias, we allocate different numbers of generated samples per category, with normal scenes receiving the largest portion and rare anomaly types receiving fewer but representative samples. This strategy ensures both diversity and balance in the textual corpus.

\noindent\textbf{Temporal Prompt Design.}
To mimic the temporal structure of video events, we design a prompt instructing the LLM to produce sequential descriptions:
\begin{quote}

\textit{``You are a professional video event timeline designer. Generate 4 sequential sentences describing a [CLASS] event.''}

\end{quote}
\textit{Each sample thus contains four time-ordered sentences covering the \textit{beginning}, \textit{development}, \textit{climax}, and \textit{resolution} of the event, and the action is introduced with where, when, and how it happens to reflect real-world visual dynamics.} This four-step structure offers concise yet expressive temporal evolution while avoiding unnecessary verbosity. We repeatedly apply the prompt with randomized sampling parameters to encourage diversity in vocabulary, narrative detail, and contextual cues.




\noindent\textbf{Quality Control and Text Augmentation.}
We remove samples with incomplete sentences, redundancy, or temporal inconsistencies. To enrich linguistic diversity, concise event captions from~\cite{liu2025surveillancevqa} (unrelated to the target domain) are expanded into sequential descriptions using the same prompt, introducing additional variation while preserving anomaly-related semantics.

In summary, category-aware sampling, temporally structured prompts, and quality control collectively yield a diverse, video-like textual corpus that effectively substitutes video data for training. Details are provided in Appendix~\ref{app:sec:more:results}.

To ensure temporal coherence between textual descriptions and visual streams, we proportionally repeat each description to a fixed length of 16 time steps during text encoding. The resulting 16-step text embeddings are then zero-padded to 256 dimensions, enabling frame-level alignment with video representations.

\subsection{Text-only Training of VAD Model}

Training a vision-free VAD model requires learning temporal anomaly patterns without access to any video data. This is made possible by the pre-trained CLIP model, whose image and text encoders project visual frames and textual descriptions into a shared semantic space. As a result, text descriptions that depict the temporal evolution of an event can serve as effective surrogates for video sequences, with each sentence functioning analogously to a video frame.

\begin{assumption}[Bounded Cross-Modal Metric Distortion]
\label{assump:alignment}
We assume that the pre-trained CLIP model induces a shared vision--language embedding space
with bounded cross-modal distortion at the instance level.

Let $\phi_{\mathrm{t}}(\cdot)$ and $\phi_{\mathrm{v}}(\cdot)$ denote the frozen text and image encoders.
For any semantic concept $c\in\mathcal{C}$, let $S\sim P_{\mathrm{t}}(\cdot\mid c)$ and
$V\sim P_{\mathrm{v}}(\cdot\mid c)$ denote random text and video samples associated with concept $c$.
There exists a constant $\varepsilon>0$ such that
\begin{equation}
\label{eq:instance_alignment}
\Pr\!\left(
\left\|
\phi_{\mathrm{t}}(S)-\phi_{\mathrm{v}}(V)
\right\|_2
\le \varepsilon
\;\middle|\; c
\right)
\ge 1-\rho,
\end{equation}
for some small $\rho\in(0,1)$. More details in Assumption ~\ref{app:assump:alignment}
\end{assumption}

\noindent\textbf{Text Embedding and Training Input.}
Under Assumption~\ref{assump:alignment}, the text-only training set is defined as 
$\mathcal{D}=\{(S_i, Y_i)\}_{i=1}^M$, where
$S_i = \{s_\tau\}_{\tau=1}^{T_s}$ denotes the text descriptions,
$s_\tau$ denotes the sentence at position $\tau$, and
$Y_i \in \{0,1\}^K$ is its multi-hot label.  
Each description is encoded by the frozen CLIP text encoder
\begin{align}
\phi_{\mathrm{t}} : S \rightarrow \mathbb{R}^{T_s \times d},
\quad
F_i^\star = \phi_{\mathrm{t}}(S_i), \quad \forall i \in [M],
\end{align}
yielding a sequence of $T_s$ text-derived embeddings.
After applying the temporal alignment operator $\mathcal{A}$, the
training input to the temporal predictor is given by
$F_i = \mathcal{A}(F_i^\star) \in \mathbb{R}^{T\times d}$.
Although purely textual, these descriptions exhibit a clear and well-structured temporal progression, thereby motivating the need for subsequent temporal modeling to capture their sequential dependencies.

\noindent\textbf{Event Evolution Causal Attention.}
To capture the temporal structure inherent in the generated text sequences, we introduce the Event-evolution Causal Attention (EC-Attn) module. It consists of two complementary components:

\noindent\textit{(1) Event-Focus Causal Attention (EFC-Attn).}
LLM-generated descriptions typically contain fine-grained action cues that correspond to short, localized temporal patterns within an event.  
To effectively capture such local dynamics, EFC-Attn divides the sequence into fixed, non-overlapping windows and applies masked self-attention within each window using a causal constraint (\ie, a lower-triangular attention mask).  
This windowed causal attention encourages the model to focus on the immediate temporal neighborhood, producing localized representations that preserve short-range temporal order and highlight subtle action transitions, which can be formulated as follows:
\begin{align}
\mathbf{H}^{\mathrm{ins}} = \operatorname{EFC-Attn}(\phi_{t}(S_{i}))~, \quad S_{i} = \{\bm{s}_\tau\}_{\tau=1}^{T_{s}}~.
\end{align}

\noindent\textit{(2) Event-Context Causal Attention (ECC-Attn).}
While EFC-Attn focuses on localized temporal cues, many anomalous events evolve gradually and require reasoning over the full sequence.  
To capture such long-range dependencies without breaking temporal causality, ECC-Attn performs global attention modulated by a learnable causal bias.
Given queries $\mathbf{Q}\in\mathbb{R}^{T\times d_k}$, keys $\mathbf{K}\in\mathbb{R}^{T\times d_k}$, 
and values $\mathbf{V}\in\mathbb{R}^{T\times d_v}$, the standard attention score matrix 
$\mathbf{S}\in\mathbb{R}^{T\times T}$ is computed as:
$\mathbf{S} = \mathbf{Q}\mathbf{K}^\top / \sqrt{d_k}$.
To impose soft causal ordering, a distance-aware bias 
$\mathbf{B}\in\mathbb{R}^{T\times T}$ is introduced.
This bias encourages attention toward nearby history while still allowing the model to access more distant past information, and strictly prevents attending to the future.
The global causal attention weights are computed as:
\begin{equation}
\mathbf{A} = \operatorname{softmax}(\mathbf{S} + \mathbf{B}),
\qquad \mathbf{A}\in\mathbb{R}^{T\times T},
\end{equation}
and the resulting context-enhanced representation is obtained as follows:
\begin{align}
\mathbf{H}^{\mathrm{evo}} = \operatorname{ECC-Attn}(\phi_{t}({S}_{i}))~, \quad S_{i} = \{\bm{s}_\tau\}_{\tau=1}^{T_{s}}~.
\end{align}

\noindent\textbf{Event Representation Integration and Hierarchical Classification.}
The Event-Focus and Event-Context causal attention modules yield two complementary 
representations: an instant-level representation 
$\mathbf{H}^{\mathrm{ins}}$ that captures short-term motion primitives, and an 
evolution-level representation $\mathbf{H}^{\mathrm{evo}}$ that models the long-range 
temporal evolution of an event.  
To consolidate these two temporal perspectives in a lightweight yet expressive manner, 
we derive a unified event representation through element-wise integration:
\begin{equation}
\mathbf{H}^{\mathrm{evt}} = \mathbf{H}^{\mathrm{inst}} + \mathbf{H}^{\mathrm{evo}}~.
\end{equation}
This unified representation retains both fine-grained temporal details and holistic 
event-evolution structure, yielding a more discriminative description of the underlying event.

The event representation $\mathbf{H}^{\mathrm{evt}}$ is then passed to a hierarchical anomaly-aware classification branch, which includes a binary head for anomaly detection 
and a multi-class head for distinguishing among specific anomaly types for yielding prediction results $\{p_\tau\}_{\tau=1}^{T_{s}}$ and $\{z_\tau\}_{\tau=1}^{T_{s}}$, respectively.  
By jointly modeling instant-level cues and long-range event evolution, our text-only framework effectively learns anomaly patterns without requiring any visual supervision.

\subsection{Optimization Objective}
The model is trained under a composite objective that jointly optimizes coarse-grained anomaly detection and fine-grained event recognition for every part of generated text descriptions, guided by a multiple-instance learning paradigm to navigate temporal ambiguity. 

\noindent\textbf{Binary Anomaly Detection.}
For the $i$-th text description, the temporal predictor outputs anomaly
scores $\{p_\tau^{\,i}\}_{\tau=1}^{T_{s}}$ after temporal alignment.
Following a multiple-instance learning strategy, we select the top-$k_i$
most salient positions:
\begin{equation}
\mathcal{I}_i
=
\operatorname*{arg\,top}_{|\mathcal{I}|=k_i}
\{p_\tau^{\,i}\}_{\tau=1}^{T_{s}},
\qquad
\bar{p}_i
=
\frac{1}{k_i}
\sum\nolimits_{\tau\in\mathcal{I}_i} p_\tau^{\,i}.
\end{equation}
The aggregated score $\bar{p}_i$ is supervised by a binary label
$y_i\in\{0,1\}$ derived from $Y_i$ using the binary cross-entropy:
\begin{equation}
\mathcal{L}_{\mathrm{bce}}
\!=\!
-\frac{1}{M}
\sum\nolimits_{i=1}^{M}
\big[
y_i \log \bar{p}_i
\!+\!
(1-y_i)\log(1-\bar{p}_i)
\big].
\end{equation}
This formulation emphasizes the most anomalous temporal segments while
being robust to irrelevant text snippets.

\noindent\textbf{Multi-Class Event Classification.}
The predictor outputs logits $z_{\tau,k}^{\,i}$ for each temporal position
$\tau$ and anomaly concept $c_k\in\mathcal{C}$.
Following the same top-$k_i$ aggregation strategy, the aggregated logits
are computed as
\begin{equation}
\bar{z}_k^{\,i}
=
\frac{1}{k_i}
\sum\nolimits_{\tau \in \mathcal{I}_i}
z_{\tau,k}^{\,i},
\qquad
\hat{y}_k^{\,i}
=
\frac{\exp(\bar{z}_k^{\,i})}
{\sum_{k'=1}^{K}\exp(\bar{z}_{k'}^{\,i})}.
\end{equation}
The multi-class classification loss is defined as
\begin{equation}
\mathcal{L}_{\mathrm{cls}}
=
-\frac{1}{M}
\sum\nolimits_{i=1}^{M}
\sum\nolimits_{k=1}^{K}
y_k^{\,i}
\log \hat{y}_k^{\,i},
\end{equation}
where $y^{\,i}\in\{0,1\}^{K}$ indicates the ground-truth anomaly concept.

\subsection{Inference for Videos}
After training a VAD model using video-like textual descriptions generated by an LLM, our target is to transfer the VAD model trained only with text-only data to the vision modality. Therefore, the encoded text embeddings should match the corresponding video frames. Thanks to the pre-trained contrastive model CLIP, which aligns the text embedding and image embedding, we can directly feed the embeddings from video frames ${v}_i$ into the VAD model trained by the text embedding. Therefore, to reduce the influence of the modality gap, we use the frozen CLIP image encoder $\phi_{\mathrm{v}}(\cdot)$ to extract video frame features.
Then, the obtained image embeddings are fed into the trained VAD model to compute the anomaly scores of all test video sequences. The frame-level anomaly degree can be obtained via the multi-class classifier. Specifically, subtracting the softmax probability of normal by one is the anomaly degree. The proposed approach retains the visual information of the image embedding for anomaly detection. When adapting to different datasets, we can simply replace the classes in the text prompt with the classes of the target dataset.

\section{Experiment}
\subsection{Experimental Setup}
\textbf{Datasets.} We evaluate the proposed approach on two large-scale benchmarks for the VAD task, \ie, \textbf{UCF-Crime}~\cite{sultani2018real} (UCF) and \textbf{XD-Violence}~\cite{wu2020not} (XD). As our method does not use any video data for training, we adopt their official test set to evaluate our method. Specifically, the UCF dataset contains 290 test videos (140 abnormal and 150 normal) with 13 anomaly types. As for XD, this dataset includes 800 test videos (500 abnormal and 300 normal) from six categories.

\begin{table}[t]
\renewcommand{\arraystretch}{0.8}
\centering
\caption{Comparison with weakly-supervised, one-class, unsupervised, and vision-free video anomaly detection methods on XD-Violence. The best results are \textbf{bolded}.}

\resizebox{0.95\linewidth}{!}{
\begin{tabular}{cccc} 
\toprule
Method & Backbone & AP(\%) & AUC(\%) \\ 

CLIP~\cite{radford2021learning} & ViT & 17.83 & 38.21  \\ 

\midrule
\multicolumn{4}{c}{\textbf{Weakly-supervised VAD}} \\ \midrule 
Wu et al.~\cite{wu2020not} & C3D-RGB & 67.19 & -- \\
Wu et al.~\cite{wu2020not} & I3D-RGB & 73.20 & --  \\
MSL~\cite{li2022self} & C3D-RGB & 75.53 & --  \\
\midrule
\multicolumn{4}{c}{\textbf{One-class VAD}} \\ \midrule 
HASAN et al.~\cite{hasan2016learning} & $\text{AE}^{\text{RGB}}$ & -- & 50.32   \\
LU et al.~\cite{lu2013abnormal} &Dictionary  & -- & 53.56   \\
BODS~\cite{wang2019gods} & I3D-RGB & -- & 57.32 \\
GODS~\cite{wang2019gods} & I3D-RGB & -- & 61.56 \\
\midrule
\multicolumn{4}{c}{\textbf{Unsupervised VAD}} \\ \hline 
RareAnom~\cite{thakare2023rareanom} & I3D-RGB & -- & 68.33   \\
\midrule
\rowcolor{gray!20}\textbf{Ours} & ViT & \textbf{75.83} &  \textbf{89.50} \\ 
\bottomrule
\end{tabular}
}
\label{tab:xd}
\vspace{-4mm}
\end{table}

\noindent \textbf{Metrics.} Following prior work~\cite{wang2025learning}, we report the area under the frame-level ROC curve (AUC) for both evaluation datasets and the frame-level average precision (AP) for the XD dataset.

\noindent \textbf{Baselines.} We compare with representative methods at each supervision level: weakly-supervised~\cite{sultani2018real, zhang2019temporal, zaheer2022generative, wu2020not}, one-class~\cite{hasan2016learning, lu2013abnormal, wang2019gods}, unsupervised~\cite{thakare2023rareanom}.

\begin{table}[t]
\renewcommand{\arraystretch}{0.9}
\caption{Comparison with weakly-supervised, one-class, unsupervised, and vision-free video anomaly detection methods on UCF-Crime. The best results are \textbf{bolded}.}

\centering
\resizebox{0.90\linewidth}{!}{
  \begin{tabular}{ccc}  
    \toprule
    Method & Backbone & AUC(\%) \\  
    CLIP~\cite{radford2021learning} & ViT & 53.16 \\
    \midrule
    \multicolumn{3}{c}{\textbf{Weakly-supervised VAD}} \\ \midrule  
    SULTANI et al.~\cite{sultani2018real} & C3D-RGB & 75.41 \\
    SULTANI et al.~\cite{sultani2018real} & I3D-RGB & 77.92 \\
    IBL~\cite{zhang2019temporal} & C3D-RGB & 78.66 \\
    GCL~\cite{zaheer2022generative} & ResNext  & 79.84 \\
    \multicolumn{3}{c}{\textbf{One-class VAD}} \\ \hline  
    SVM~\cite{sultani2018real} & - & 50.00 \\
    SSV~\cite{sohrab2018subspace} & - & 58.50 \\
    BODS~\cite{wang2019gods} & I3D-RGB & 68.26 \\
    GODS~\cite{wang2019gods} & I3D-RGB & 70.46 \\ \midrule
    \multicolumn{3}{c}{\textbf{Unsupervised VAD}} \\ \midrule  
    GCL~\cite{zaheer2022generative} & ResNext & 74.20 \\
    TUR et al.~\cite{tur2023exploring} & ResNet  & 65.22 \\
    TUR et al.~\cite{tur2023unsupervised} & ResNet & 66.85 \\
    DYANNET~\cite{thakare2023dyannet} & I3D & 79.76 \\ \hline
    \rowcolor{gray!20} \textbf{Ours} & ViT & \textbf{80.82} \\ 
    \bottomrule
  \end{tabular}
  }
  \label{tab:ucf}
  \vspace{-4mm}
\end{table}

\begin{table*}
\centering
\caption{Real-time performance comparison with the LAVAD method. The best results are \textbf{bolded}.} 

\resizebox{0.8\linewidth}{!}{
\begin{tabular}{ccccccccccc}
\toprule
\multirow{2}{*}{Method} & \multirow{2}{*}{} & \multicolumn{3}{c}{XD-Violence} &    & \multicolumn{1}{c}{UCF-Crime} &    & Frame rate   &  & Params Size    \\ \cline{3-5} \cline{7-7}
                        &                   & AP(\%)     &      & AUC(\%)     &    & AUC(\%)     &    & (frames/sec) &  &  \\ \hline

LAVAD~\cite{zanella2024harnessing}               &                   & 62.01      &      & 85.36          &    & 80.28          &    & 1.26        &  & 13B             \\
\rowcolor{gray!20}Ours                   &                   & \textbf{75.83}          &       & \textbf{89.50}    &  & \textbf{80.82}    &  & \textbf{183}         &  & \textbf{0.31B}        \\ \bottomrule
\end{tabular}
}
\vspace{-4mm}
\label{tab:real_time}

\end{table*}

\noindent \textbf{Implementation.} 
For the model structure, frozen image and text encoders are adopted from pre-trained CLIP (ViT-B/16). We generate 17674 and 3253 text descriptions for XD and UCF datasets, respectively. For the local window length in the local-action causal attention module, we set it as 4 and 16 for XD and UCF datasets, respectively. We set top-$k$ as 16 in loss function computation.
For training, the proposed TD-VAD is trained on a single RTX 4090 using the PyTorch framework. We use AdamW as the optimizer with a batch size of 64. On XD-Violence, the learning rate and total epoch are set as $3\times10^{-5}$ and 10, respectively, and on UCF-Crime, the learning rate and total epoch are set as $2\times10^{-4}$ and 10, respectively.

\subsection{Comparison Results}
We undertake a comprehensive evaluation of the proposed TD-VAD's anomaly detection capabilities by benchmarking it against widely-used approaches across XD-Violence and UCF-Crime datasets in Table~\ref{tab:xd} and Table~\ref{tab:ucf}, respectively. 

For the XD-Violence dataset, TD-VAD's AP on this dataset are 8.64\% higher than the advanced weakly-supervised method~\cite{wu2020not}, even though no visual data is used. Also, our method outperforms the baseline CLIP~\cite{radford2021learning} (38.21 vs. 75.83). The proposed method (AUC 89.50) also surpasses strong unsupervised (vs. 68.33), one-class (vs. 61.56), and several weakly-supervised baselines. Table~\ref{tab:ucf} also shows the competitive anomaly detection performance of the proposed approach on the UCF-Crime dataset. We can observe that the performance of our method on this dataset is 5.41\% higher than the weakly-supervised method SULTANI et al.~\cite{sultani2018real}. At the same time, our method achieves an impressive 80.82 AUC, making a significant 2.16 absolute improvement compared to the advanced method~\cite{zhang2019temporal}. Also, the proposed TD-VAD outperforms all unsupervised and one-class methods, as well as some weakly-supervised methods, on UCF-Crime.  
Different from the above methods, our approach adopts the pure text data generated by LLM for training. These presented results demonstrate our approach's superiority and robustness even if no visual data is available. The season for such performance we analyzed is that the generated text precisely captures the essential semantics and temporal logic of abnormal events. Leveraging CLIP’s cross-modal alignment and EC-Attn’s structured temporal modeling, TD-VAD learns universal anomaly discrimination rules from text, which effectively generalize to real videos.

\begin{table}[t]
\centering
\caption{Cross-dataset results on UCF-Crime and XD-Violence.}
\resizebox{0.80\linewidth}{!}{
\begin{tabular}{c|c|cc}
\toprule
Test$\Rightarrow$ & \multicolumn{1}{c|}{UCF-Crime} & \multicolumn{2}{c}{XD-Violence} \\
Train$\downarrow$ & AUC(\%)  & AP(\%) & AUC(\%) \\
\midrule
UCF-Crime & \cellcolor{gray!20}\textbf{80.82}  & 81.39 & 59.53\\
XD-Violence & 77.08 & \cellcolor{gray!20}\textbf{75.83} & \cellcolor{gray!20}\textbf{89.50} \\
\bottomrule
\end{tabular}%
}
\label{tab:cross_dataset}
\vspace{-4mm}
\end{table}

As shown in Table~\ref{tab:real_time}, we compare the proposed TD-VAD approach with LAVAD~\cite{zanella2024harnessing}, which also leverages the off-the-shelf LLMs and VLMs and follows a language-driven pathway for estimating the anomaly scores, in terms of accuracy, frame rate, and parameter size. Even though the proposed TD-VAD has significantly fewer model parameters (including the frozen CLIP encoder) than LAVAD, its performance shows a substantial improvement compared to LAVAD. This can significantly reduce computational and deployment costs for the model. Moreover, on the XD dataset, TD-VAD improves AP from LAVAD's 62.01 to 75.83, while boosting frame rate from 1.26 FPS to 183 FPS, a speed-up of roughly 145 times. TD-VAD achieves a higher frame rate, lower parameter size, and higher accuracy compared to LAVAD, thereby demonstrating the effectiveness of the proposed approach.

\textbf{Cross-Dataset Validation}. In Table~\ref{tab:cross_dataset}, we conduct the cross-dataset experiment to evaluate the generalization of the proposed TD-VAD model. For the scenario where the trained TD-VAD model is tested on an unseen dataset, TD-VAD still achieves competitive performance. This indicates that the semantic patterns learned from generated text descriptions can generalize to distinct scene distributions, verifying the cross-dataset transferability of our text-driven paradigm. Additionally, such performance also validates that the proposed TD-VAD approach can detect truly unseen or unknown anomalies.

\begin{figure*}[!t]
  \centering
   \includegraphics[width=1\linewidth]{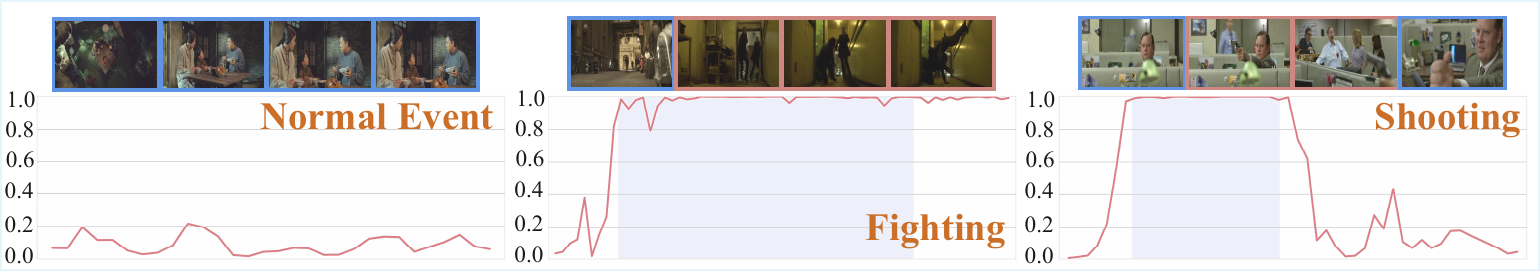}  
   \caption{Qualitative visualization results of the proposed TD-VAD approach. As illustrated in this figure, the red curves represent the anomaly scores, and the blue regions denote the ground-truth abnormal temporal location. As we can see, the proposed approach detects abnormal regions of different categories precisely. Meanwhile, it also produces low anomaly scores on purely normal video sequences.  }
   \label{fig:example}
\end{figure*}

\subsection{Ablation Study}
We conduct extensive ablation studies to validate the effectiveness of the key components in our method on the XD-Violence dataset.

\noindent \textbf{Timestamps of Text Descriptions.} To investigate the influence of timestamps of generated text descriptions, we conduct experiments to generate text descriptions with different timestamps. The experimental results are provided in Table~\ref{subtab:timestamp}. We find that increasing the timestamps of generated text descriptions can effectively improve the performance in training. It is worth noting that the 4 timestamps achieve better results than the 8 timestamps. The reason we analyzed is that the text descriptions of anomaly events with 4 timestamps align more closely with the actual anomaly boundaries, yielding smaller semantic gaps.

\noindent \textbf{Window Length in ECC-Attn.} We explore the influence of different window lengths in the proposed event-focus causal attention. 
As shown in the table~\ref{subtab:length}, the window length set as 4 achieves the best performance. Additionally, from the table, we find that even with different window lengths, the TD-VAD still performs well, which demonstrates the insensitivity to this hyper-parameter.

\begin{table}[ht]
  \centering
  \renewcommand{\arraystretch}{0.9} 
  \caption{Ablation studies of timestamps of generated text descriptions, and the window length in ECC-Attn.}
 
  \resizebox{0.9\linewidth}{!}{%
    \subcaptionbox{Timestamps\label{subtab:timestamp}}{%
      \begin{tabular}[t]{c|cc}
        \toprule
        Timestamps & AP    & AUC   \\ \midrule
        1          & 71.79 & 87.57 \\
        2          & 72.91 & 87.94 \\
        \rowcolor{gray!20}\textbf{4} & \textbf{75.83} & \textbf{89.50} \\
        8          & 74.67 & 88.89 \\ \bottomrule
      \end{tabular}%
    }
    \hspace{2em} 
    \subcaptionbox{Window Length\label{subtab:length}}{%
      \begin{tabular}[t]{c|cc}
        \toprule
        Length & AP    & AUC   \\ \midrule
        2      & 75.24 & 89.08 \\
        \rowcolor{gray!20} \textbf{4}      & \textbf{75.83} & \textbf{89.50} \\
        8      & 74.98 & 89.43 \\
        16     & 75.12 & 89.57 \\ \bottomrule
      \end{tabular}%
    }%
  }
  \label{tab:ablation_timestamps_length}
  \vspace{-4mm}
\end{table}

\begin{table}[]
\centering
\caption{Ablation studies of the proposed event evolution causal attention module. The best results are \textbf{bolded}.}

\resizebox{0.80\linewidth}{!}{ 
\begin{tabular}{cccc|ccc}
\toprule
ECC-Attn &  & EFC-Attn &  & AP &  & AUC \\ \midrule
  -  &  & -    &  & 30.48   &  & 60.33    \\
  $\checkmark$  &  & -    &  &  71.29  &  &  88.18  \\
  -  &  &  $\checkmark$   &  & 73.09   &  &  87.84   \\
 \rowcolor{gray!20}$\checkmark$   &  &  $\checkmark$   &  & \textbf{75.83}   &  &  \textbf{89.50}   \\ \bottomrule
\end{tabular}
}\label{tab:causal}
\vspace{-4mm}
\end{table}

\begin{figure}[t]
  \centering
   \includegraphics[width=0.9\linewidth]
   {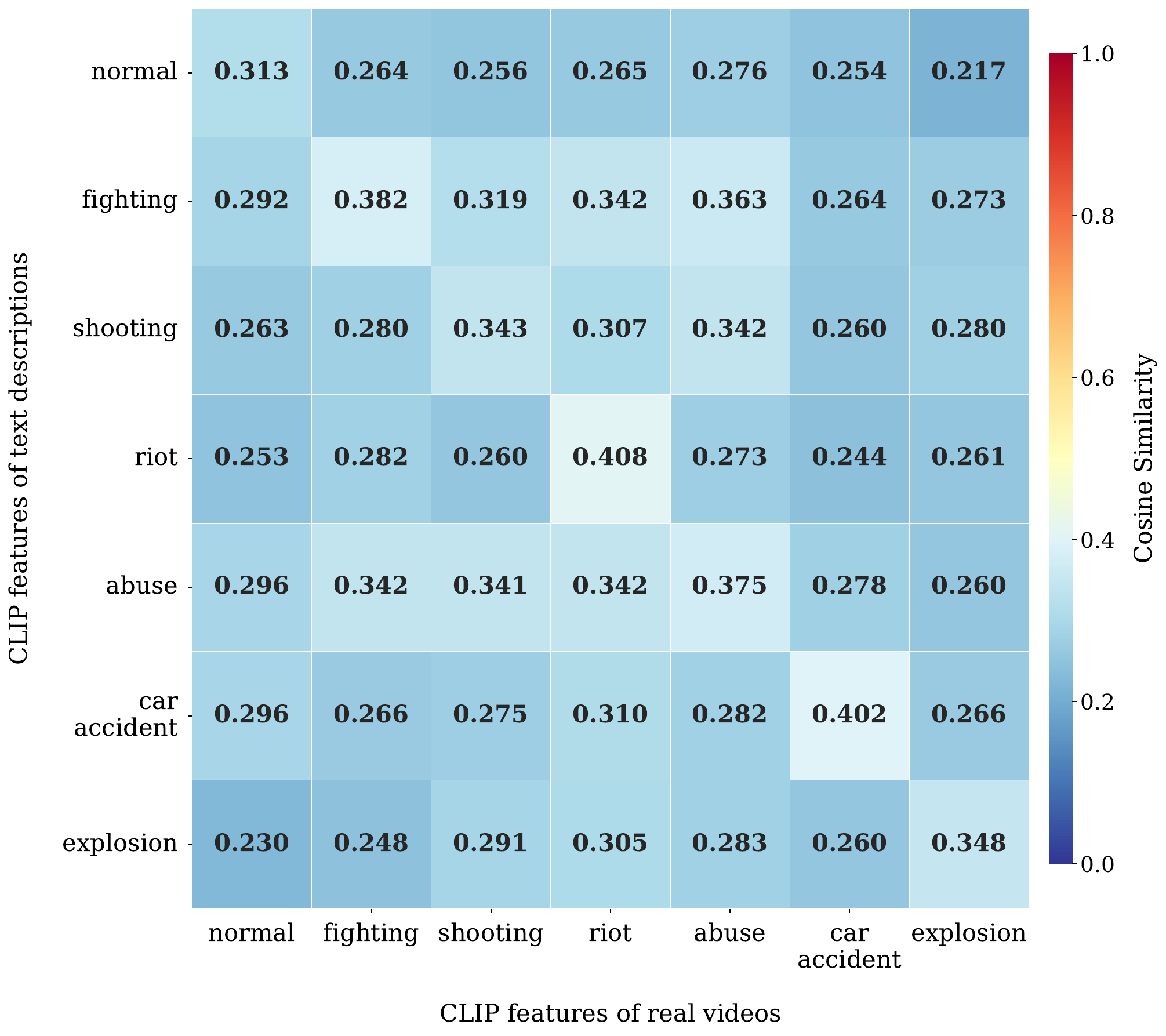}
  
   \caption{Similarity Matrix of CLIP Cross-Modal Features: Text Descriptions vs. Real Videos from XD-Violence dataset.  }
   \vspace{-2mm}
   \label{fig:tsne}
   
\end{figure}

\begin{table}[]

\centering
\renewcommand{\arraystretch}{0.8}
\caption{Ablation studies of the proposed hierarchical anomaly-aware classifier module. The best results are \textbf{bolded}.}

\resizebox{0.80\linewidth}{!}{ 
\begin{tabular}{cccc|ccc}
\toprule
Binary &  & Multi-class &  & AP &  & AUC \\ \midrule
  $\checkmark$  &  & -    &  &  68.22  &  &  86.77  \\
  -  &  &  $\checkmark$   &  & 74.83   &  &  89.35   \\
 \rowcolor{gray!20}$\checkmark$   &  &  $\checkmark$   &  & \textbf{75.83}   &  &  \textbf{89.50}   \\ \bottomrule
\end{tabular}
}\label{tab:classifier}

\end{table}

\noindent \textbf{Event-Evolution Causal Attention Module}. The ablation results in table~\ref{tab:causal} demonstrate that combining EFC-attn with the ECC-attn mechanism yields strong performance gains. The EFC-attn captures local dynamics of events, while the ECC-attn mechanism models long-range dependencies. Their integration enables the VAD model to learn both short dynamics and global event patterns of anomaly events, which is crucial to the overall performance improvement. This highlights the complementary roles of the two components and the effectiveness of their joint design.

\noindent \textbf{Hierarchical Anomaly-aware Branch}.  In Table~\ref{tab:classifier}, we conduct an in-depth analysis to examine the impact of the binary and multi-class classifiers in the proposed hierarchical anomaly-aware branch on anomaly detection performance on XD-Violence. With the assistance of feature optimization in the binary branch, the multi-class classifier obtains a performance improvement compared to the previous one.  Experimental results clearly show that a dual approach can boost the performance by leveraging the complementarity of different granularities.

\subsection{Qualitative Analyses}
\label{sec:Analyses}

\noindent \textbf{Qualitative Visualization}. Fig.~\ref{fig:example} shows the qualitative results of the proposed TD-VAD with sample videos from XD-Violence. Upon examination, it becomes evident that the proposed TD-VAD is capable of detecting anomaly events accurately. Additionally, our method also produces considerably low anomaly scores in normal areas. The visualization results demonstrate that the proposed TD-VAD approach can effectively detect anomaly events in video sequences.

\textbf{Visualization of Similarity between Text and Visual Features}. As shown in Fig.~\ref{fig:tsne}, for each event category, the similarity between intra-class text and video features is significantly higher than inter-class pairs. This confirms that text descriptions accurately capture the core semantic features of corresponding real-world anomalies. However, \textbf{one limitation of the generated text corpus} is that due to the shared violent semantics and overlapping visual manifestations, the relatively high cross-category similarity between text features from the fighting event and video features from the abuse event may confuse fine-grained classification. This is a common limitation for fine-grained VAD task.

\section{Conclusion}
To break the dependence on annotated video data, a novel TD-VAD approach is proposed only using text data. We propose a new view of treating text descriptions as videos for training a VAD model in a vision-free manner. The LLM is exploited to efficiently generate a high-quality textual dataset for training. To address the modality gap between the video and text, we design a simple yet effective video-like text generation method and utilize the aligned space of the CLIP model. Extensive experimental results demonstrate the effectiveness of the proposed TD-VAD. In the future, we will continue to explore the vision-language pre-trained knowledge and further devote to the vision-free VAD task.

\section*{Acknowledgements}

This work was supported by the Fundamental and Interdisciplinary Disciplines Breakthrough Plan of the Ministry of Education of China (JYB2025XDXM902), the National Natural Science Foundation of China (No. 62476124, 62276134), the Natural Science Foundation of Jiangsu Province (No. BK20242015), the Gusu Innovation and Entrepreneur Leading Talents (No. ZXL2025322), and Nanjing University - China Mobile Communications Group Co., Ltd. Joint Institute.

\section*{Impact Statement}

This paper presents work whose goal is to advance the field of 
Machine Learning. There are many potential societal consequences 
of our work, none of which we feel must be specifically highlighted here.

\nocite{langley00}

\bibliography{example_paper}
\bibliographystyle{icml2026}

\newpage
\appendix
\onecolumn









\section{Preliminaries}\label{app:sec:preli}

In Section~\ref{app:sec:preli:basic}, we introduce basic notation used throughout the appendix.
In Section~\ref{app:sec:preli:problem}, we formally define the vision-free VAD problem and the
surrogate-training setup.




\subsection{Notations}
\label{app:sec:preli:basic}

For a positive integer $n$, let $[n]:=\{1,2,\ldots,n\}$.
We use $k\in[K]$ to index anomaly concepts in
$\mathcal{C}=\{c_k\}_{k=1}^K$.
The Euclidean norm is denoted by $\|\cdot\|_2$.
A video is represented as a frame sequence
$V=\{v_t\}_{t=1}^{T_v}$, and a surrogate text sequence as
$S=\{s_\tau\}_{\tau=1}^{T_s}$, where $T_v$ and $T_s$ need not be equal.
Let $\phi_{\mathrm{v}}(\cdot)$ and $\phi_{\mathrm{t}}(\cdot)$ denote the
frozen CLIP image and text encoders, mapping inputs to a
$d$-dimensional embedding space.
The resulting embedding sequences are
$X^\star=\phi_{\mathrm{v}}(V)\in\mathbb{R}^{T_v\times d}$ and
$F^\star=\phi_{\mathrm{t}}(S)\in\mathbb{R}^{T_s\times d}$.
Both sequences are mapped to a common length $T$ using a deterministic
temporal alignment operator $\mathcal{A}$, yielding
$X=\mathcal{A}(X^\star)\in\mathbb{R}^{T\times d}$ and
$F=\mathcal{A}(F^\star)\in\mathbb{R}^{T\times d}$.
We denote by $g_\theta:\mathbb{R}^{T\times d}\to[0,1]^T$ a temporal
predictor.
At the population level, $(F,Y)$ and $(X,Y)$ are random variables drawn
from the surrogate and target distributions $P_S$ and $P_T$,
respectively.

\subsection{Problem Definition and Assumption }\label{app:sec:preli:problem}

We formalize the vision-free VAD setting studied in this work: the learner has no access to
target-domain training videos and must rely on semantic anomaly concepts only.
To enable learning without visual supervision, we train a temporal predictor on surrogate text
sequences in a frozen CLIP embedding space and deploy it on videos via modality substitution.
We next provide the formal definitions used in our theoretical analysis.

\begin{definition}[Vision-Free Video Anomaly Detection]
\label{def:vision_free_vad}
In the \emph{vision-free} setting, no target-domain training videos are available, i.e.,
$D_{\mathrm{train}}=\varnothing$.
Instead, the learner is given only a set of semantic anomaly concepts
$\mathcal{C}=\{c_k\}_{k=1}^K$.
Given an unseen test video $V_{\mathrm{test}}=\{v_t\}_{t=1}^{T_v}$,
the goal is to produce frame-level anomaly scores
\[
\varphi_\theta:\ (V_{\mathrm{test}},\mathcal{C})
\;\mapsto\;
\{p_t\}_{t=1}^{T_v} \in [0,1]^{T_v}.
\]
Here $\varphi_\theta$ denotes the end-to-end anomaly scoring function induced by the
temporal predictor $g_\theta$ together with the frozen CLIP encoders and temporal
alignment operator.
\end{definition}

\begin{definition}[Surrogate Training and Modality Substitution]
\label{def:surrogate_substitution}
Let $S=\{s_\tau\}_{\tau=1}^{T_s}$ denote a surrogate text sequence and
$V_{\mathrm{test}}=\{v_t\}_{t=1}^{T_v}$ a test video, where $T_s$ and $T_v$ need not be equal.
We assume that both sequences are mapped to a common length $T$ via a deterministic
temporal alignment operator (e.g., repetition or interpolation).

Let $\phi_{\mathrm{t}}(\cdot)$ and $\phi_{\mathrm{v}}(\cdot)$ be the frozen CLIP text and image encoders.
Training is performed using aligned surrogate embeddings
$F=\phi_{\mathrm{t}}(S)\in\mathbb{R}^{T\times d}$, and inference is performed by applying the same
predictor to aligned video embeddings
$X=\phi_{\mathrm{v}}(V_{\mathrm{test}})\in\mathbb{R}^{T\times d}$.

We train a temporal predictor $g_\theta:\mathbb{R}^{T\times d}\to[0,1]^T$ on surrogate embeddings,
and deploy it on video embeddings via modality substitution, yielding
$p=g_\theta(X)\in[0,1]^T$.
\end{definition}

\begin{assumption}[Bounded Cross-Modal Metric Distortion]
\label{app:assump:alignment}
We assume that the pre-trained CLIP model induces a shared vision--language embedding space
with bounded cross-modal distortion at the instance level.

Let $\phi_{\mathrm{t}}(\cdot)$ and $\phi_{\mathrm{v}}(\cdot)$ denote the frozen text and image encoders.
For any semantic concept $c\in\mathcal{C}$, let $S\sim P_{\mathrm{t}}(\cdot\mid c)$ and
$V\sim P_{\mathrm{v}}(\cdot\mid c)$ denote random text and video samples associated with concept $c$.
There exists a constant $\varepsilon>0$ such that
\begin{equation}
\label{eq:instance_alignment}
\Pr\!\left(
\left\|
\phi_{\mathrm{t}}(S)-\phi_{\mathrm{v}}(V)
\right\|_2
\le \varepsilon
\;\middle|\; c
\right)
\ge 1-\rho,
\end{equation}
for some small $\rho\in(0,1)$.

Moreover, for any pair of concepts $c,c'\in\mathcal{C}$ and independent samples
$S\sim P_{\mathrm{t}}(\cdot\mid c)$, $S'\sim P_{\mathrm{t}}(\cdot\mid c')$,
$V\sim P_{\mathrm{v}}(\cdot\mid c)$, $V'\sim P_{\mathrm{v}}(\cdot\mid c')$, we have
\begin{equation}
\label{eq:pairwise_distortion}
\big|
\|\phi_{\mathrm{t}}(S)-\phi_{\mathrm{t}}(S')\|_2
-
\|\phi_{\mathrm{v}}(V)-\phi_{\mathrm{v}}(V')\|_2
\big|
\le 2\varepsilon
\end{equation}
with probability at least $1-2\rho$.
\end{assumption}

\begin{remark}[Relation to PAC consistency~\cite{valiant1984theory}]
The bounded distortion condition in Assumption~\ref{assump:alignment} resembles a
PAC-style consistency assumption at the representation level, where $\varepsilon$
controls the accuracy of cross-modal alignment and $\rho$ the failure probability.
Unlike classical PAC learning, we do not make claims about sample complexity or
empirical risk minimization; instead, the assumption serves to justify generalization
under modality substitution.
\end{remark}

\section{Theoretical Analysis}\label{app:sec:theoreticalanalysis}

This section explains why \emph{text-only} training can generalize to videos in the vision-free setting.
Our analysis follows three steps: (i) we view text embeddings as \emph{surrogate data} and relate the
target (video) risk to the surrogate (text) risk via a standard domain-discrepancy decomposition;
(ii) we characterize the effect of \emph{modality substitution} (replacing text embeddings by video-frame
embeddings at test time) through a stability bound, and (iii) we discuss additional gaps introduced by
surrogate sequence construction (LLM generation and temporal alignment).
The key enabling condition throughout our analysis is the
\emph{bounded cross-modal metric distortion} induced by the pre-trained CLIP embedding space,
formalized in Assumption~\ref{assump:alignment}.%
\footnote{Assumption~\ref{assump:alignment} mirrors Assumption 3.1 in the main body, which motivates
text-only training and deployment on videos through CLIP alignment.}
This assumption states that text and video instances sharing the same semantic concept are mapped
to nearby points in the embedding space, while relative distances between different concepts are
approximately preserved across modalities.

\subsection{Surrogate-to-Target Generalization via Domain Discrepancy}
\label{app:subsec:surrogate_generalization}

Recall that $F=\phi_{\mathrm{t}}(S)$ and $X=\phi_{\mathrm{v}}(V)$ denote the surrogate (text-derived) and
target (video-derived) sequence embeddings in the shared CLIP space (Section~\ref{app:sec:preli:basic}).
Let $P_S$ and $P_T$ be the distributions over $(F,Y)$ and $(X,Y)$, respectively, and let $\ell\in[0,1]$
be a bounded loss. For any hypothesis $g\in\mathcal{G}$, define the risks
\[
R_S(g):=\mathbb{E}_{(F,Y)\sim P_S}\big[\ell(g(F),Y)\big],\qquad
R_T(g):=\mathbb{E}_{(X,Y)\sim P_T}\big[\ell(g(X),Y)\big].
\]

\paragraph{A Concrete Discrepancy.}
To quantify the surrogate--target mismatch, we instantiate $\mathrm{Div}_{\mathcal{G}}(P_S,P_T)$ using an
integral probability metric (IPM)-style discrepancy~\cite{muller1997integral}.
Under the bounded cross-modal metric distortion Assumption~\ref{app:assump:alignment}, text-derived and video-derived
representations corresponding to the same semantic concepts are expected to induce similar prediction statistics, leading to a small surrogate--target discrepancy in the CLIP embedding space.

\begin{definition}[$\mathcal{G}$-IPM discrepancy]\label{def:ipm}
Let $\mathcal{G}$ be a class of measurable predictors $g:\mathbb{R}^{T\times d}\to[0,1]^T$.
Define
\begin{align}
\mathrm{IPM}_{\mathcal{G}}(P_S,P_T)
:=\sup_{g\in\mathcal{G}}
\left|
\mathbb{E}_{(F,Y)\sim P_S}\!\left[\ell(g(F),Y)\right]
-
\mathbb{E}_{(X,Y)\sim P_T}\!\left[\ell(g(X),Y)\right]
\right|.
\end{align}
\end{definition}

\begin{proposition}[Surrogate-to-Target Risk Decomposition]
\label{prop:surrogate_to_target}
Fix a hypothesis class $\mathcal{G}$ and a bounded loss $\ell\in[0,1]$.
For any $g\in\mathcal{G}$,
\begin{equation}
\label{eq:surrogate_bound}
R_T(g)
\;\le\;
R_S(g)
\;+\;
\mathrm{IPM}_{\mathcal{G}}(P_S,P_T)
\;+\;
\lambda^\star,
\end{equation}
where $\lambda^\star:=\inf_{h\in\mathcal{G}}\big(R_S(h)+R_T(h)\big)$ is the shared optimal joint error.
\end{proposition}

\begin{proof}
Let
$g^\star \in \arg\min_{h\in\mathcal{G}} \big(R_S(h)+R_T(h)\big).$
We start by adding and subtracting the surrogate risk term:
\[
R_T(g)
= R_S(g) + \big(R_T(g)-R_S(g)\big).
\]
By Definition~\ref{def:ipm} (the $\mathcal{G}$-IPM discrepancy), for any fixed $g\in\mathcal{G}$ we have
\[
\big|R_T(g)-R_S(g)\big|
=
\left|
\mathbb{E}_{(X,Y)\sim P_T}[\ell(g(X),Y)]
-
\mathbb{E}_{(F,Y)\sim P_S}[\ell(g(F),Y)]
\right|
\le \mathrm{IPM}_{\mathcal{G}}(P_S,P_T).
\]
Hence,
$
R_T(g)
\le R_S(g) + \mathrm{IPM}_{\mathcal{G}}(P_S,P_T).
$
To introduce the shared optimal joint error $\lambda^\star$, note that
\[
\lambda^\star
:= \inf_{h\in\mathcal{G}} \big(R_S(h)+R_T(h)\big)
= R_S(g^\star)+R_T(g^\star),
\]
and since risks are nonnegative, $\lambda^\star \ge 0$. Therefore, adding $\lambda^\star$ to the
right-hand side preserves the inequality:
\[
R_T(g)
\le R_S(g) + \mathrm{IPM}_{\mathcal{G}}(P_S,P_T) + \lambda^\star.
\]
This is exactly \eqref{eq:surrogate_bound}.
\end{proof}

\paragraph{Interpretation for TD-VAD.}
Eq.~\eqref{eq:surrogate_bound} formalizes \emph{text as surrogate data}: minimizing surrogate risk
$R_S$ yields small target risk $R_T$ when (i) the surrogate and target distributions are close in the
CLIP embedding space (small $\mathrm{IPM}_{\mathcal{G}}$), and (ii) there exists a predictor that performs
well on both domains (small $\lambda^\star$).
Both conditions are enabled by the bounded cross-modal metric distortion of CLIP, which limits
semantic and geometric mismatch between text and video representations.
In TD-VAD, both modalities are embedded by frozen CLIP encoders,
and inference explicitly feeds video-frame embeddings into the text-trained model.

\subsection{Stability of Modality Substitution}
\label{app:subsec:substitution_stability} 

Unlike the surrogate-to-target analysis in
Section~\ref{app:subsec:surrogate_generalization}, which relates
distribution-level risks, this subsection focuses on an
instance-level stability property of the learned predictor.
Specifically, we analyze the effect of replacing the surrogate
(text-derived) embeddings with target (video-derived) embeddings at
test time and show that, under bounded cross-modal metric distortion,
modality substitution can be interpreted as a controlled perturbation
in the representation space.

\begin{definition}[Concept-wise alignment error]\label{def:alignment_error}
Let $\mu_{\mathrm{t}}(c):=\mathbb{E}[F\mid c]$ and
$\mu_{\mathrm{v}}(c):=\mathbb{E}[X\mid c]$ denote the class-conditional
prototypes of text and video embeddings, respectively, in the shared
CLIP space. We define the alignment error as
\begin{align*}
\varepsilon := \sup_{c\in\mathcal{C}}
\left\| \mu_{\mathrm{t}}(c)-\mu_{\mathrm{v}}(c) \right\|_2 .
\end{align*}
\vspace{-1.0em}
\end{definition}

The alignment error quantifies the maximum discrepancy between text
and video representations of the same semantic concept and provides a
measure of cross-modal mismatch at the prototype level.

\begin{proposition}[Modality Substitution Stability]
\label{prop:substitution_stability}
Assume that the temporal predictor $g_\theta$ is $L$-Lipschitz with
respect to its sequence input, i.e., for all
$A,B\in\mathbb{R}^{T\times d}$,
\[
\|g_\theta(A)-g_\theta(B)\|_2 \le L\|A-B\|_2 .
\]
If a test instance satisfies $\|F-X\|_2 \le \varepsilon$, then
\[
\|g_\theta(X)-g_\theta(F)\|_2 \le L\varepsilon .
\]
\end{proposition}

\begin{proof}
The result follows directly from the Lipschitz continuity of
$g_\theta$ by setting $(A,B)=(X,F)$.
\end{proof}

\paragraph{Discussion.}
Proposition~\ref{prop:substitution_stability} establishes a
pointwise stability guarantee: when text-derived and video-derived
embeddings corresponding to the same semantic concept are close in the
representation space, replacing $F$ with $X$ induces a bounded change
in the model predictions. This result does not rely on any
distributional assumptions and complements the surrogate-to-target
generalization analysis in
Section~\ref{app:subsec:surrogate_generalization} by justifying the
modality substitution step at the functional level.

\subsection{Surrogate Construction Gaps}
\label{app:subsec:surrogate_gaps}

Beyond cross-modal alignment, the surrogate domain itself is
\emph{constructed} rather than naturally observed.
In TD-VAD, temporally ordered text descriptions are generated by an
LLM and subsequently transformed to a fixed length for frame-level
processing.
This construction process introduces additional sources of mismatch
that are not captured by cross-modal alignment alone.

To disentangle these effects, we conceptually decompose the
surrogate-to-target gap as
\[
R_T(g)
\;\le\;
R_S(g)
+
\Delta_{\mathrm{gen}}
+
\Delta_{\mathrm{temp}}
+
\Delta_{\mathrm{align}} .
\]
This decomposition serves as an analytical tool rather than a tight
generalization bound.

\paragraph{Cross-modal alignment error.}
The term $\Delta_{\mathrm{align}}$ captures residual mismatch between
text-derived and video-derived representations in the shared CLIP
embedding space and is controlled by the alignment error $\varepsilon$
(Definition~\ref{def:alignment_error}) under the bounded cross-modal
metric distortion assumption (Assumption~\ref{assump:alignment}).
Its effect on predictions is analyzed in
Section~\ref{app:subsec:substitution_stability}.

\paragraph{Temporal alignment error.}
Let $F^\star \in \mathbb{R}^{T^\star \times d}$ denote the original
text-derived embedding sequence prior to temporal alignment, and let
$\mathcal{A}$ denotes the deterministic alignment operator that maps
$F^\star$ to a fixed-length sequence $F=\mathcal{A}(F^\star)$.
We define the temporal alignment error as
\[
\Delta_{\mathrm{temp}}
:=
\mathbb{E}\!\left[
\ell\big(g(\mathcal{A}(F^\star)),Y\big)
-
\ell\big(g(F^\star),Y\big)
\right],
\]
which quantifies the distortion introduced by enforcing frame-level
temporal structure through repetition, padding, or interpolation.

\paragraph{Generation error.}
Let $P_t^\star$ denote an idealized text distribution that perfectly
captures the semantic and temporal variability of video events, and
let $P_t^{\mathrm{LLM}}$ denote the distribution induced by LLM-based
text generation.
We define the generation error as
\[
\Delta_{\mathrm{gen}}
:=
\sup_{g\in\mathcal{G}}
\left|
\mathbb{E}_{(F,Y)\sim P_t^{\mathrm{LLM}}}
[\ell(g(F),Y)]
-
\mathbb{E}_{(F,Y)\sim P_t^\star}
[\ell(g(F),Y)]
\right|,
\]
which captures the distributional bias introduced by the text generation
process.

\paragraph{Practical implication.}
This decomposition clarifies that, while cross-modal alignment and
modality substitution stability governs representation transfer. Overall performance in the vision-free setting is also influenced by
temporal alignment design and the quality of the LLM-generated surrogate
data.

\section{More Results}\label{app:sec:more:results}
In this supplementary material, we provide the exact form of the prompt details for LLM to generate text descriptions for events, class distributions, and representative examples of the generated text corpus. 


\textbf{Prompt Detail}. To mimic the temporal structure of video events, we design a prompt instructing the LLM to produce sequential descriptions. The prompt differs for XD-Violence~\cite{wu2020not} and UCF-Crime~\cite{sultani2018real}. For XD-Violence, the prompt is structured as:
\begin{quote}
\textit{``You are a professional video event timeline designer. Generate 4 sequential sentences describing a [CLASS] event. Follow these strict rules: 1. Each sentence must depict a completely unique visual scene with no repeated actions, objects, or elements from previous descriptions; 2. Timeline logic must be natural: start$\rightarrow$development$\rightarrow$climax$\rightarrow$resolution (for abnormal events) or causal sequence (for normal events); 3. Use present tense and concrete language - describe only what can be visually observed; 4. Each sentence must be 15-25 words long with a varied structure. 5. Include specific details (colors, locations, actions) to ensure uniqueness; 6. Never reuse phrases from the examples provided in the template.''}
\end{quote}
Additionally, for UCF-Crime, the prompt has a similar form, which is structured as follows:
\begin{quote}
\textit{``Generate 4 sequential surveillance camera descriptions of a '[CLASS]' event. Ensure each description is 20-35 words from a security camera perspective. Create logical progression: detection$\rightarrow$incident development$\rightarrow$climax$\rightarrow$resolution. Include surveillance-specific details: camera angle, monitoring coverage, and visibility limitations. Focus on visible behavior only - no internal thoughts or off-camera events. Use surveillance terminology: CCTV, security footage, monitoring system, camera view. Make this sample unique for surveillance training data.''}
\end{quote}

%

\textbf{Prompt Sensitivity Analysis}. We evaluate the proposed TD-VAD approach's robustness to various prompt formulations while keeping the core information unchanged. This reveals whether LLM depends on specific prompt engineering or has genuine semantic understanding for anomalous events. We test three prompt styles: (1) Law Enforcement uses professional surveillance terminology; (2) Education adopts explanatory language for teaching; (3) Domain expert incorporates specialized security vocabulary. As shown in Table~\ref{tab:prompt_sensitivity_ucf}, performance varies within $\pm1\%$ AUC across all styles, demonstrating robust stability. These results validate our prompt design for generating anomalous descriptions via LLM while conforming to stability across prompt formulations.   
\begin{table}[h]
\centering
\caption{Prompt sensitivity analysis on XD-Violence dataset. Baseline uses the video event timeline designer style.}
\begin{tabular}{cc}
\toprule
Prompt Style          & AUC Change (\%) \\
\midrule
Law Enforcement Style    & -0.50           \\
Educational Style     & +0.02           \\
Domain Expert Style   & -0.97          \\
\bottomrule
\end{tabular}

\label{tab:prompt_sensitivity_ucf}
\end{table}

\textbf{Robustness to Noisy Descriptions}. To validate the robustness to noisy text descriptions, we randomly remove 10\%–40\% of temporal text segments during training to simulate incomplete and missing descriptions from LLMs. As shown in Table~\ref{tab:temporal_text_ablation}, interestingly, the TD-VAD model achieves slightly better performance rather than degradation. This demonstrates that our approach is highly robust to incomplete textual narratives. The model focuses on learning the core semantic concepts of anomalies instead of relying on full temporal details. The mild information dropout acts as gentle regularization that reduces overfitting to text patterns, leading to improved generalization on the test set.

\begin{table}[h]
\centering
\caption{Effect of removing varying percentages of temporal text on the XD-Violence dataset. Values in parentheses indicate the change compared to the baseline.}
\begin{tabular}{lcc}
\toprule
& AUC & AP \\
\midrule
Remove 10\%  & 89.65(+0.15) & 76.19(+0.36) \\
Remove 20\%  & 89.51(+0.01) & 75.84(+0.01) \\
Remove 30\%  & 89.65(+0.15) & 76.18(+0.35) \\
Remove 40\%  & 88.98(-0.52) & 74.32(-1.51) \\
\bottomrule
\end{tabular}

\label{tab:temporal_text_ablation}
\end{table}

\textbf{Effect of Different LLMs}. To investigate the influence of different text data generation models, we further used Kimi-k2.5~\cite{team2026kimi} and GPT-3~\cite{brown2020language} to generate the text dataset. As shown in Table~\ref{tab:llm_comparison}, we find that the model based on the text dataset generated by DeepSeek-V3 used in the manuscript works best. Specifically, both Kimi-k2.5 and GPT-3 lead to performance drops compared with our baseline, demonstrating that DeepSeek-V3~\cite{liu2024deepseek} is more capable of producing high-quality, semantically accurate, and video-aligned text descriptions. The relatively superior performance of GPT-3 over Kimi-k2.5 also indicates that the logical expression and contextual modeling capabilities of LLMs affect the quality of generated text supervision. These results validate that choosing an appropriate LLM for text dataset construction is crucial in our task.

\begin{table}[h]
\centering
\caption{Performance comparison of different large language models on the XD-Violence dataset. Values in parentheses indicate the change compared to the DeepSeek-V3 baseline.}
\begin{tabular}{ccc}
\toprule
LLM & AUC & AP \\
\midrule
Kimi-k2.5~\cite{team2026kimi} & 89.02(-0.48) & 74.50(-0.33) \\
GPT-3~\cite{brown2020language}     & 89.42(-0.08) & 75.54(-0.29) \\
\bottomrule
\end{tabular}

\label{tab:llm_comparison}
\end{table}

\textbf{Effect of Different Vision-Language Models}. To investigate the influence of different vision-language models (VLMs) for alignment across video and text modalities, we evaluate our approach with CLIP and two SigLIP backbones~\cite{zhai2023sigmoid} on XD-Violence in Table~\ref{tab:encoder_comparison}. While SigLIP excels on general image-text benchmarks via its pairwise sigmoid loss, CLIP's global softmax contrastive alignment better captures the fine-grained, long-range visual-text correspondences critical for temporal anomaly detection. This explains why CLIP achieves higher performance in our video anomaly detection task.  

\begin{table}[h]
\centering
\caption{Performance comparison of different VLMs on the XD-Violence dataset. Values in parentheses indicate the change compared to the CLIP baseline.}
\begin{tabular}{ccc}
\toprule
VLM & AUC & AP \\
\midrule
CLIP~\cite{radford2021learning} & 89.50 & 75.83 \\
Siglip-base-patch16-224~\cite{zhai2023sigmoid} & 87.19 (-2.31) & 67.55 (-8.28) \\
Siglip-base-patch16-512~\cite{zhai2023sigmoid} & 89.18 (-0.32) & 71.84 (-3.99) \\
\bottomrule
\end{tabular}

\label{tab:encoder_comparison}
\end{table}

\textbf{Limitation Analysis}. We present the per-category AUC performance using the multi-class anomaly-aware branch in Table~\ref{tab:per_class_auc}. Our approach achieves consistent performance across most anomaly categories, showing reliable capability in detecting anomalies. However, as discussed in Section~\ref{sec:Analyses} and illustrated in Fig.~\ref{fig:tsne}, \textit{Abuse} and \textit{Shooting} share highly similar violent semantics and visual patterns with other anomalous events, leading to higher fine-grained classification difficulty. This high semantic overlap creates a more challenging distinction for the proposed TD-VAD approach, resulting in relatively lower raw performance for these two fine-grained anomaly categories. In future work, we plan to design a dedicated cross-modal alignment paradigm tailored for video anomaly detection to further refine fine-grained subtle event modeling.   

\begin{table}[h]
\centering
\caption{Per-class AUC performance on the XD-Violence dataset.}
\begin{tabular}{cc}
\toprule
Category & AUC \\
\midrule
Normal        & 89.22 \\
Fighting      & 86.46 \\
Shooting      & 72.68 \\
Riot          & 96.10 \\
Abuse         & 61.47 \\
Car Accident  & 84.67 \\
Explosion     & 95.03 \\
\bottomrule
\end{tabular}

\label{tab:per_class_auc}
\end{table}

\textbf{Visualization of Text Features}. We visualize the feature distribution of generated text descriptions by using t-SNE, and present results in Fig.~\ref{fig:multi-tsne}. As we can see, the features of text descriptions generated by LLM have distinguishable boundaries, show features of texts that describe a certain type of anomaly event, and exhibit categorical discriminability. It demonstrates the effectiveness of using text data generated in training.

\textbf{Class Distribution and Examples for Text Corpus}. In Table~\ref{xd_counting} and Table~\ref{ucf_counting}, we present the detailed amounts of generated text descriptions for different categories in the XD-Violence and UCF-Crime datasets, respectively. We allocate more generated samples for usual anomaly categories, yet the rare anomaly types receive fewer but representative samples. This strategy ensures both diversity and balance in the textual corpus to reflect the distribution of real-world anomalies. 
\begin{table}[ht]
  \centering
    \caption{The amount of generated text descriptions for different categories in XD-Violence and UCF-Crime datasets}
  \makebox[0.95\linewidth][c]{%
    \subcaptionbox{XD-Violence dataset\label{xd_counting}}{%
      \renewcommand{\arraystretch}{0.8} 
      \resizebox{0.40\linewidth}{!}{
        \begin{tabular}{c|c || c|c}
          \toprule
          Category     & Amount &  Category   & Amount \\ 
          \midrule
          Normal       & 12450  & Shooting     & 559 \\
          Car accident & 2204   & Explosion    & 505 \\
          Fighting     & 1066   & Riot         & 200 \\
          Abuse        & 690    & -            & -    \\
          \bottomrule
        \end{tabular}%
      }%
    }
    \hfill 
    \subcaptionbox{UCF-Crime dataset\label{ucf_counting}}{%
      \renewcommand{\arraystretch}{0.9} 
      \resizebox{0.40\linewidth}{!}{
        \begin{tabular}{c|c || c|c}
          \toprule
          Category       & Amount &  Category   & Amount \\ 
          \midrule
          Normal         & 483    &  Shoplifting  & 50 \\
          Road accidents & 1756   &  Arrest       & 60 \\
          Explosion      & 282    &  Assault      & 60 \\
          Fighting       & 135    &  Vandalism    & 60 \\
          Shooting       & 90     &  Burglary     & 50 \\
          Abuse          & 77     &  Robbery      & 50 \\
          Arson          & 60     &  Stealing     & 40 \\         
          \bottomrule
        \end{tabular}%
      }%
    }%
  }
  \vspace{0.5em} 

  \label{tab:dataset_text_amount} 
\end{table}

\begin{figure}[t]
  \centering
   \includegraphics[width=0.5\linewidth]
   {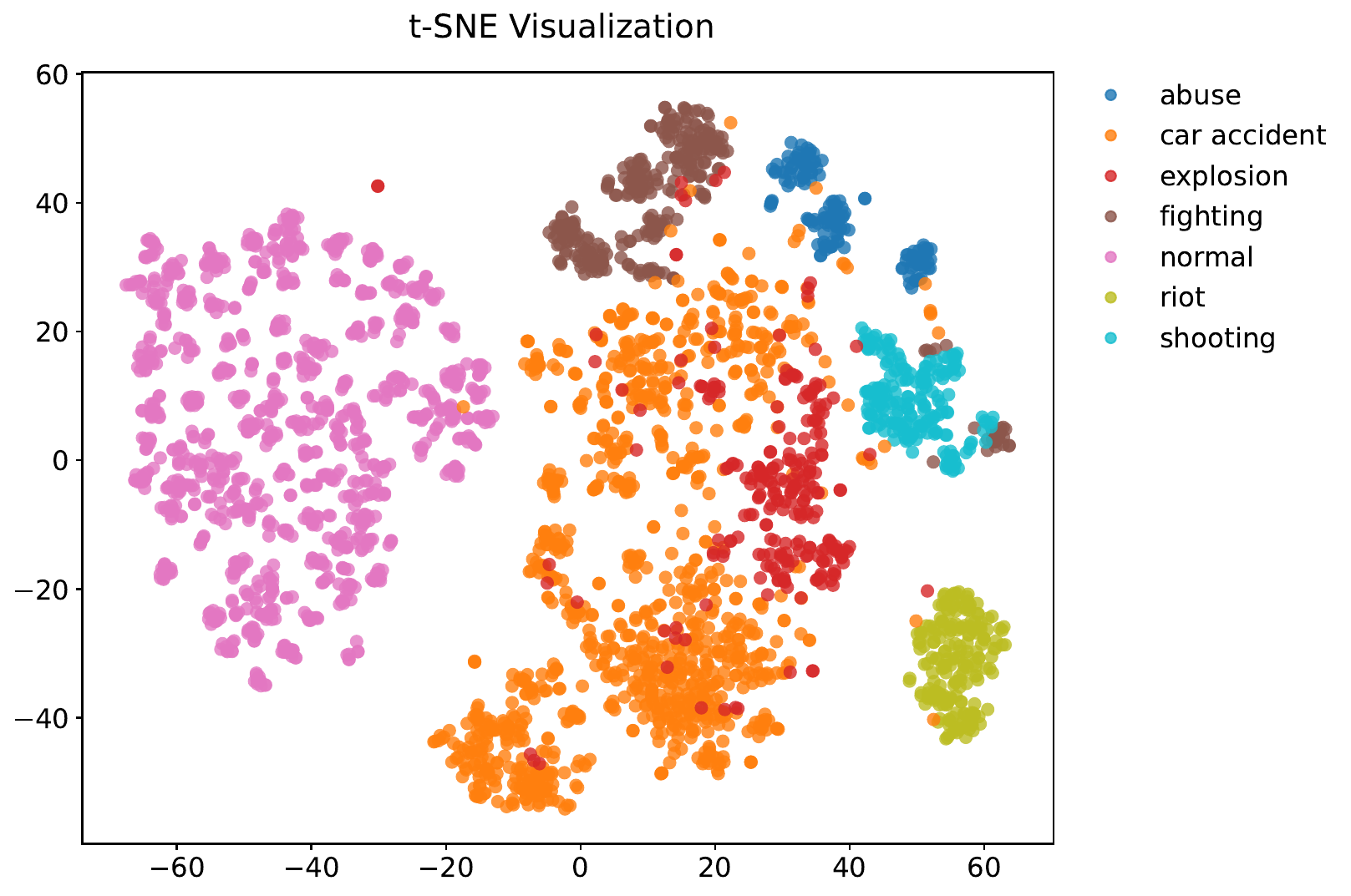}
   \caption{The Visualization of the normal and abnormal embeddings of generated text descriptions in 2D space by t-SNE.  }
   \label{fig:multi-tsne}
\end{figure}

In Fig.~\ref{fig:xd_example} and Fig.~\ref{fig:ucf_example}, we present the examples of generated text descriptions via LLM in the XD-Violence and UCF-Crime dataset, respectively. These generated examples exhibit a clear temporal progression to mimic the temporal structure of different video events. \textit{As illustrated in Fig.~\ref{fig:xd_example} and Fig.~\ref{fig:ucf_example}, the generated text descriptions are not limited to describing actions alone, which inherently encode the coupling of actions, scenes, and spatial-temporal contexts, which forces the model to learn “context-dependent anomaly” rather than “action-only anomaly”.} For example, a clear illustration in in Fig.~\ref{fig:xd_example} lies in the identical run action across our anomalous and normal text descriptions: the abuse-related text “She breaks free and runs toward the main street's bright lights as bystanders begin gathering near the alley's entrance" embeds run in abnormality-indicative context cues—the subject’s "breaking free" implies a prior threatening interaction, the dark alley connotes a high-risk scenario, and bystanders’ gathering reflects a response to deviance. By contrast, the normal text “As evening approaches, two children run laughing through the sprawling green park while their parents watch from a bench" denotes the same run action with normality signatures, where the tranquil park setting, children’s joyful affect, and parental supervision collectively denote a routine, non-deviant scenario.



\begin{figure*}[t]
  \centering

   \includegraphics[width=1\linewidth]{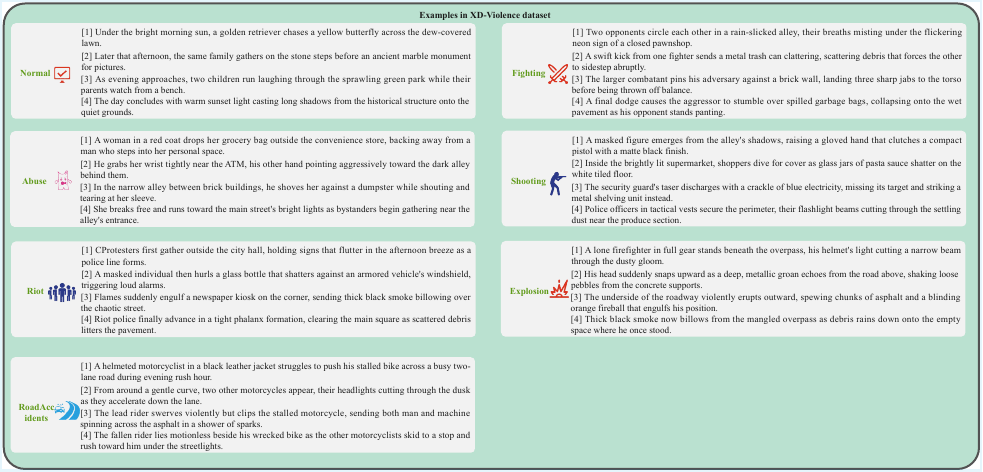}

   \caption{Examples of generated text descriptions via LLM in the XD-Violence dataset. }
   \label{fig:xd_example}
\end{figure*}

\begin{figure*}[t]
  \centering

   \includegraphics[width=1\linewidth]{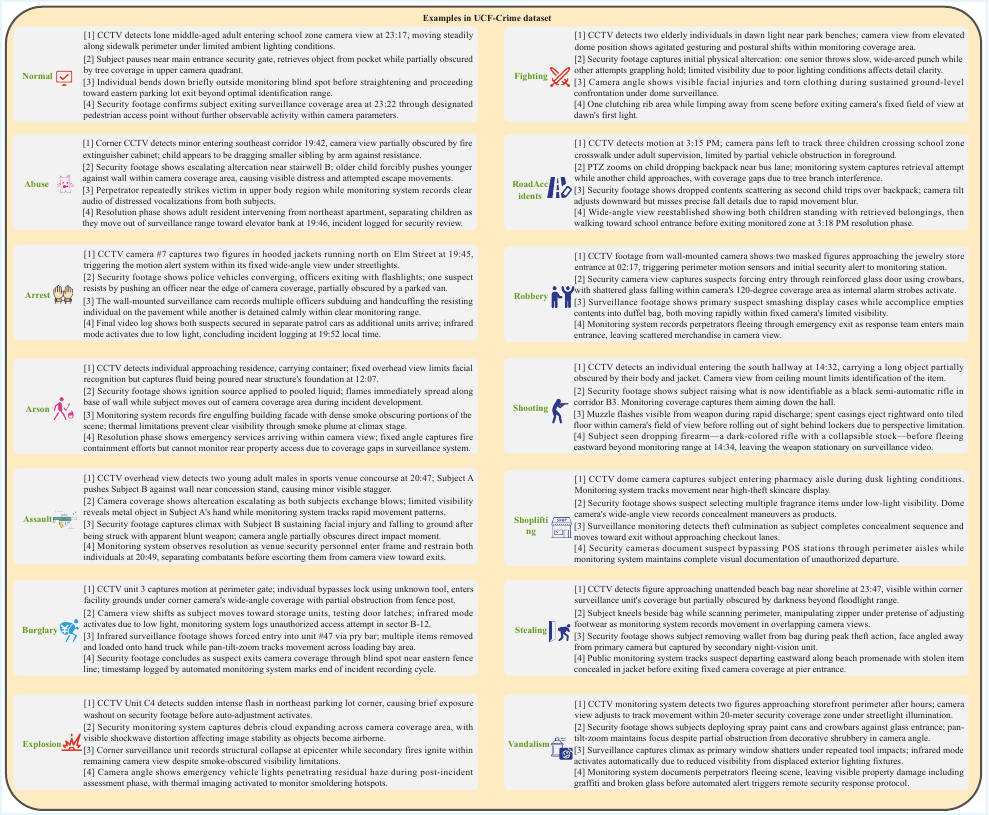}

   \caption{Examples of generated text descriptions via LLM in the UCF-Crime dataset. }
   \label{fig:ucf_example}
\end{figure*}


\end{document}